\documentclass[runningheads]{llncs}
\usepackage[T1]{fontenc}
\usepackage{graphicx}
\usepackage{subfig}
\usepackage{capt-of}  
\usepackage{cuted}
\usepackage[backend=biber, sorting=none, giveninits=true]{biblatex}
\usepackage{graphicx}%
\usepackage{multirow}%
\usepackage{amsmath,amssymb,amsfonts}%
\usepackage{mathrsfs}%
\usepackage[figuresright]{rotating}%
\usepackage{xcolor}%
\usepackage{textcomp}%
\usepackage{manyfoot}%
\usepackage{booktabs}%
\usepackage{algorithm}%
\usepackage{algorithmicx}%
\usepackage{algpseudocode}%
\usepackage{program}%
\usepackage{listings}%
\usepackage{wrapfig}
\usepackage{vruler}
\usepackage{setspace}
\usepackage{comment}
\usepackage{siunitx}
\usepackage{booktabs}
\usepackage{CJKutf8}

\usepackage{algorithm}
\usepackage{multicol}
\usepackage{paralist}

\begin{document}
    
%
\title{Is Complexity an Illusion?}

%
%
\author{Michael Timothy Bennett\inst{1}\\\orcidID{0000-0001-6895-8782} 
}
\authorrunning{Michael Timothy Bennett}
%
\institute{The Australian National University\\
\email{michael.bennett@anu.edu.au}\\
\url{http://www.michaeltimothybennett.com/}}
\maketitle              

\begin{abstract} 
Simplicity is held by many to be the key to general intelligence. Simpler models tend to “generalise”, identifying the cause or generator of data with greater sample efficiency. The implications of the correlation between simplicity and generalisation extend far beyond computer science, addressing questions of physics and even biology. Yet simplicity is a property of form, while generalisation is of function. In interactive settings, any correlation between the two depends on interpretation. In theory there could be no correlation and yet in practice, there is. Previous theoretical work showed generalisation to be a consequence of “weak” constraints implied by function, not form. Experiments demonstrated choosing weak constraints over simple forms yielded a $110-500\%$ improvement in generalisation rate. Here we show that all constraints can take equally simple forms, regardless of weakness. However if forms are spatially extended, then function is represented using a finite subset of forms. If function is represented using a finite subset of forms, then we can force a correlation between simplicity and generalisation by making weak constraints take simple forms. If function is determined by a goal directed process that favours versatility (e.g. natural selection), then efficiency demands weak constraints take simple forms. Complexity has no causal influence on generalisation, but appears to due to confounding. 
\end{abstract}

\section{Introduction}
Complexity is a quality of systems we find difficult to understand. Formal analogues include entropy \cite{sep-information-entropy}, compression \cite{kolmogorov1963,rissanen1978} and even fractal dimension \cite{barnsley1993a}. Physicist Leonard Susskind believes complexity may be the key to a unified theory of physics \cite{gefter2014,susskind2014}. Cyberneticist Francis Heylighen recently argued that goals are attractors of dynamical systems that self organise in complex reaction networks \cite{heylighen2023b}. If complex reaction networks do self organise as he argues \cite{heylighen2008}, then it goes some way towards explaining the origins of goal directed behaviour, and thus life. Finally, many hold that complexity is the key to general intelligence \cite{legg2007}. Language models like GPT-4 amount to compressed representations of human language \cite{deletang2024}. Simpler objects can be compressed to greater extents, because they exhibit self similarity. As a result, some hold that compression is general intelligence, meaning a general reinforcement learning agent like AIXI \cite{hutter2010} can use Solomonoff Induction \cite{solomonoff1978} to maximise expected reward across a wide range of environments.
Yet for all that complexity seems to be at the heart of every matter, it has profound flaws. As a qualitative indicator of how \textit{subjectively} difficult a system might be to understand, it makes perfect sense. It makes far less as an indicator of anything \textit{objective}. Ockham's Razor is the epistemic principle that simpler statements are more likely to hold true. It can be understood as the claim that our subjective perceptions of complexity reflect an objective property of our environment. There is no obvious reason this should be the case, and yet it is \cite{sober2015}. Simpler statements tend to be more accurate representations of reality. The aforementioned Solomonoff Induction formalises Ockham's Razor, meaning AIXI is based on the premise that \textit{simpler} models are more accurate depictions of the environment than complex models of seemingly equal predictive power. AIXI is a superintelligence in the sense that it maximises Legg-Hutter intelligence \cite{legg2007}. However, Jan Leike later showed that any claim regarding AIXI's performance is ``entirely subjective'' \cite{leike2015}. Legg-Hutter intelligence is measured with respect to a fixed Universal Turing Machine (UTM), and AIXI is only optimal if it uses exactly the same UTM. This calls into question the viability of complexity based induction systems in interactive settings. Their performance is subjective, and from a pragmatic standpoint it is only objective performance claims that matter. 
Leike's result suggests there could be no correlation between objective performance and subjective complexity. This concurs with what seems intuitively obvious, that complexity is an aspect of interpretation. Complexity is a measure of \textit{form}, not \textit{function}. So why does the subjective perception of simplicity tend to correlate with objective performance? 

\subsubsection{What exactly is complexity supposed to indicate?: }
As it is used in AGI research \cite{hutter2010,legg2008}, complexity is intended to help us infer the program which generated or \textit{caused} data. If one can identify that which caused past data, then one can ``generalise'' to predict the outcomes in future interactions, to maximise performance \cite{bennett2023c}. 
We are concerned with adaptation or ``the ability to generalise'' \cite{chollet2019}, not any specific circumstance. This is because any system can eventually identify cause given enough data and memory (by simply rote learning every outcome until it has a complete behavioural specification of the causal program). So assuming one \textit{can} correctly infer cause, then we claim that the amount of data one requires to do so is the sole measure of performance\footnote{Using the same dataset, not different datasets which could necessarily imply a different set of sufficient causes and thus affect learning.}. We refer to this as sample efficiency. The more sample efficiently one can infer cause, the greater one's ability to generalise and adapt to any desired end. Thus, we take intelligence to be a measure of the sample efficiency in generalisation.

\subsubsection{Key questions: }
We build upon previous work \cite{bennett2023b,bennett2023c}, in which: 
\begin{enumerate}
    \item Maximising simplicity of policies was proven unnecessary and insufficient to maximise sample efficiency \cite[prop. 3]{bennett2023b}.
    \item Maximising policy ``weakness'' was proven necessary and sufficient to maximise sample efficiency \cite[prop. 1,2]{bennett2023b} and identify cause \cite{bennett2023c}. In experiments, weak policies outperformed simple by $110-500\%$.
\end{enumerate} 
Our purpose here is to extend this work, and to establish:
\begin{enumerate}
    \item Is complexity just an artefact of abstraction?
    \item Why do sample efficiency and simplicity tend to be correlated? 
\end{enumerate}
\subsubsection{Results: }
We begin by presenting a formalism. Our results are only meaningful if one accepts that our formalism is reflective of reality, so we provide an argument to the effect that it is (lest we be accused of straw-manning complexity). \textbf{Second}, we show that the complexity of all behaviours is equal in the absence of an abstraction layer (a general formalisation of any interpreter). In other words, complexity is a subjective ``illusion''. We further show that if the vocabulary is finite then weakness can confound simplicity and sample efficiency. \textbf{Third}, we argue that abstraction is goal directed. If the vocabulary is finite, and tasks uniformly distributed, then weak statements take simple forms. 

\section{The Formalism}
The following definitions are shared with \cite{bennett2024c,bennett2024a}, which apply them to biological and philosophical perspectives.
\begin{definition}[environment]\label{environment}\hphantom{.}
\begin{itemize}{\small
    \item  We assume a set $\Phi$ whose elements we call \textbf{states}.
    \item A \textbf{declarative program} is $f \subseteq \Phi$, and we write $P$ for the set of all declarative programs (the powerset of $\Phi$).
    \item By a \textbf{truth} or \textbf{fact} about a state $\phi$, we mean $f \in P$ such that $\phi \in f$. 
    \item By an \textbf{aspect of a state} $\phi$ we mean a set $l$ of facts about $\phi$ s.t. $\phi \in \bigcap l$. By an \textbf{aspect of the environment} we mean an aspect $l$ of any state, s.t. $\bigcap l \neq \emptyset$. We say an aspect of the environment is \textbf{realised}\footnote{Realised meaning it is made real, or brought into existence.} by state $\phi$ if it is an aspect of $\phi$. 
    }
\end{itemize}

\end{definition}
\begin{definition}[abstraction layer] \label{abstractionlayer}\hphantom{.}
\begin{itemize}{\small 
    \item We single out a subset $\mathfrak{v} \subseteq P$ which we call \textbf{the vocabulary} of an abstraction layer. If $\mathfrak{v} = P$, then we say that there is no abstraction.
    \item $L_\mathfrak{v} = \{ l \subseteq \mathfrak{v} : \bigcap l \neq \emptyset \}$ is a set of aspects in $\mathfrak{v}$. We call $L_\mathfrak{v}$ a formal language, and $l \in L_\mathfrak{v}$ a \textbf{statement}.
    \item We say a statement is \textbf{true} given a state iff it is an aspect realised by that state.  
    \item A \textbf{completion} of a statement $x$ is a statement $y$ which is a superset of $x$. If $y$ is true, then $x$ is true. 
    \item The \textbf{extension of a statement}\footnote{The relation to typical philosophical and linguistic notions of intension and extension is addressed at length in \cite{bennettmaruyama2022a,bennett2022a,bennett2023d}.} $x \in {L_\mathfrak{v}}$ is $E_x = \{y \in {L_\mathfrak{v}} : x \subseteq y\}$. $E_x$ is the set of all completions of $x$. 
    \item The \textbf{extension of a set of statements} $X \subseteq {L_\mathfrak{v}}$ is $E_X = \bigcup\limits_{x \in X} E_x$. 
    \item We say $x$ and $y$ are \textbf{equivalent} iff $E_x = E_y$.
    }
\end{itemize}

\noindent {\normalfont(notation)} $E$ with a subscript is the extension of the subscript\footnote{e.g. $E_l$ is the extension of $l$.}.\\

\noindent {\normalfont(intuitive summary)} $L_\mathfrak{v}$ is everything which can be realised in this abstraction layer. The extension $E_x$ of a statement $x$ is the set of all statements whose existence implies $x$, and so it is like a truth table.\\
\end{definition}

\begin{definition}[{$\mathfrak{v}$}-task]\label{task} For a chosen $\mathfrak{v}$, a task $\alpha$ is a pair $\langle {I}_\alpha, {O}_\alpha \rangle$ where:\begin{itemize}{ \small
    \item ${I}_\alpha \subset L_\mathfrak{v}$ is a set whose elements we call \textbf{inputs} of $\alpha$. 
    \item ${O_\alpha} \subset E_{I_\alpha}$ is a set whose elements we call \textbf{correct outputs} of $\alpha$. 
} 
\end{itemize}
${I_\alpha}$ has the extension $E_{I_\alpha}$ we call \textbf{outputs}, and ${O_\alpha}$ are outputs deemed correct. 
$\Gamma_\mathfrak{v}$ is the set of \textbf{all tasks} given $\mathfrak{v}$.\\

\noindent {\normalfont(generational hierarchy)} A $\mathfrak{v}$-task $\alpha$ is a \textbf{child} of $\mathfrak{v}$-task $\omega$ if ${I}_\alpha \subset {I}_\omega$ and ${O}_\alpha \subseteq {O}_\omega$. This is written as $\alpha \sqsubset \omega$. If $\alpha \sqsubset \omega$ then $\omega$ is then a \textbf{parent} of $\alpha$. $\sqsubset$ implies a generational hierarchy of tasks. 
The level of a task $\alpha$ in this hierarchy is the largest $k$ such there is a sequence $\langle \alpha_0, \alpha_1, ... \alpha_k \rangle$ of $k$ tasks such that $\alpha_0 = \alpha$ and $\alpha_i \sqsubset \alpha_{i+1}$ for all $i \in (0,k)$. A child is ``lower level'' than its parents\footnote{Practical examples child and parent tasks are in a separately published paper with the publicly available experimental code \cite{bennett2023b}.}. \\

\noindent{\normalfont(notation)} If $\omega \in \Gamma_\mathfrak{v}$, then we will use subscript $\omega$ to signify parts of $\omega$, meaning one should assume $\omega = \langle {I}_\omega, {O}_\omega \rangle$ even if that isn't written.\\

\noindent {\normalfont(intuitive summary)} To reiterate and summarise the above:
\begin{itemize}{\small
    \item An \textbf{input} is a possibly incomplete description of a world.
    \item An \textbf{output} is a completion of an input [def.\ \ref{abstractionlayer}].
    \item A \textbf{correct output} is a correct completion of an input. 
}\end{itemize}
\end{definition}

\subsubsection{Learning and inference definitions: } 
Inference requires a \textbf{policy} and learning a policy requires a \textbf{proxy}, the definitions of which follow.

\begin{definition}[inference]\label{inference}\hphantom{.}
\begin{itemize}{\small
    \item A {$\mathfrak{v}$}-task \textbf{policy} is a statement $\pi \in L_\mathfrak{v}$. It constrains how we complete inputs.
    \item $\pi$ is a \textbf{correct policy} iff the correct outputs $O_\alpha$ of $\alpha$ are exactly the completions $\pi'$ of $\pi$ such that $\pi'$ is also a completion of an input.   
    \item The set of all correct policies for a task $\alpha$ is denoted $\Pi_\alpha$.\footnote{To repeat the above definition in set builder notation:
$${\Pi_\alpha} = \{\pi \in L_\mathfrak{v} : {E}_{I_\alpha} \cap E_{\pi} = {O_\alpha}\}$$}
}\end{itemize}
Assume $\mathfrak{v}$-task $\omega$ and a policy $\pi \in L_\mathfrak{v}$. Inference proceeds as follows:\begin{compactenum}{\small
    \item we are presented with an input ${i} \in {I}_\omega$, and
    \item we must select an output $e \in E_{i} \cap E_\pi$.
    \item If $e \in {O}_\omega$, then $e$ is correct and the task ``complete''. $\pi \in {\Pi}_\omega$ implies $e \in {O}_\omega$, but $e \in {O}_\omega$ doesn't imply $\pi \in {\Pi}_\omega$ (an incorrect policy can imply a correct output).\\}
\end{compactenum} 

\noindent {\normalfont(intuitive summary)} To reiterate and summarise the above:
\begin{itemize}{\small
    \item A \textbf{policy} constrains how we complete inputs.
    \item A \textbf{correct policy} is one that constrains us to correct outputs.
}\end{itemize}
In functionalist terms, a policy is a ``causal intermediary''. \\
\end{definition}

\begin{definition}[learning]\label{learning} \phantom{.}
\begin{itemize}{\small
    \item A \textbf{proxy} $<$ is a binary relation on statements, and the set of all proxies is $Q$.
    \item $<_w$ is the \textbf{weakness} proxy. For statements $l_1,l_2$ we have  $l_1 <_w l_2$ iff $\lvert E_{l_1}\rvert < \lvert E_{l_2} \rvert$.
    \item $<_d$ is the \textbf{description length} or \textbf{simplicity} proxy. We have  $l_1 <_d l_2$ iff $\lvert l_1 \rvert > \lvert l_2 \rvert$.}
\end{itemize}
By the \textbf{weakness} of an \textbf{extension} we mean its cardinality. By the weakness of a \textbf{statement}, we mean the cardinality of its \textbf{extension}. Likewise, when we speak of \textbf{simplicity} with regards to a \textbf{statement}, we mean its cardinality. The complexity of an \textbf{extension} is the simplicity of the simplest statement of which it is an extension\footnote{For example, if we have a language $L_\mathfrak{v}$, and $X \subset L_\mathfrak{v}$ is the set of all statements in $L_\mathfrak{v}$ that all have the extension $E_X$, then the complexity of $E_X$ is the cardinality of a statement $x \in X$ s.t. there is not statement $y \in X$ with smaller cardinality than $x$.}.\\

\noindent {\normalfont(generalisation)} A statement $l$ \textbf{generalises} to a $\mathfrak{v}$-task $\alpha$ iff $l \in \Pi_\alpha$. We speak of \textbf{learning} $\omega$ from $\alpha$ iff, given a proxy $<$, $\pi \in {\Pi}_\alpha$ maximises $<$ relative to all other policies in ${\Pi}_\alpha$, and $\pi \in {\Pi}_\omega$.\\

\noindent {\normalfont(probability of generalisation)} We assume a uniform distribution over $\Gamma_\mathfrak{v}$. 
If $l_1$ and $l_2$ are policies, we say it is less probable that $l_1$ generalizes than that $l_2$ generalizes, written $l_1 <_g l_2$, iff, when a task $\alpha$ is chosen at random from $\Gamma_\mathfrak{v}$ (using a uniform distribution) then the probability that $l_1$ generalizes to $\alpha$ is less than the probability that $l_2$ generalizes to $\alpha$. \\

\noindent {\normalfont(sample efficiency)} Suppose $\mathfrak{app}$ is the set of \textbf{a}ll \textbf{p}airs of \textbf{p}olicies. Assume a proxy $<$ returns $1$ iff true, else $0$. Proxy $<_a$ is more sample efficient than $<_b$ iff $$\left ( \sum_{(l_1,l_2) \in \mathfrak{app}} \lvert (l_1 <_g l_2) - (l_1 <_a l_2) \rvert - \lvert (l_1 <_g l_2) - (l_1 <_b l_2) \rvert  \right ) < 0$$  

\noindent {\normalfont(optimal proxy)} There is no proxy more sample efficient than $<_w$, so we call $<_w$ optimal. This formalises the idea that    
``explanations should be no more specific than necessary'' (see Bennett's razor in \cite{bennett2023b}).\\

\noindent {\normalfont(intuitive summary)} Learning is an activity undertaken by some manner of intelligent agent, and a task has been ``learned'' by an agent that knows a correct policy. Humans typically learn from ``examples''. An example of a task is a correct output and input. A collection of examples is a child task, so ``learning'' is an attempt to generalise from a child, to one of its parents. The lower level the child from which an agent generalises to parent, the ``faster'' it learns, the more sample efficient the proxy. The most sample efficient proxy is $<_w$. 
\end{definition}

\section{Arguments and Results}
The distinction between software and hardware is unsuitable for reasoning about cause \cite{bennett2023c}. The performance of a software agent in an interactive setting is \textit{subjective} \cite{leike2015}, as its behaviour depends on hardware which interprets it. Hardware is an ``abstraction layer'' between software and the surrounding environment. If we are to understand complexity, we must understand what the concept entails in the absence of such abstraction layers. We must ascertain what is \textit{objective} rather than \textit{subjective}, and so we must begin at the level of the environment rather than the agent. To this end, previous work \cite{bennett2023b,bennett2023c,bennett2023d} 
proposed a ``de facto'' pancomputational \cite{piccinini2015} model of all conceivable environments and aspects thereof, which we refine and extend. While the formalism used here is a departure from past work, it is equivalent with respect to those previously published proofs we reference \cite{bennett2023b}.
The formalism is not computational in the sense of relying on symbols, quantities, or any other high level abstraction interpreted by a human mind. The assumptions we make are extremely weak (they hold in all conceivable environments).\\

\noindent \textbf{Axiom 1: } When there are things, we call these things the environment\footnote{It might seem absurd to state something so minimal, but it is necessary to be precise about how minimal our assumptions are.}. \\


\noindent \textbf{Axiom 2: } The environment has at least one ``state''\footnote{Or if the reader prefers, there are as many environments as there are states. By ``state'' of the environment we mean the aforementioned things. If two states are not the same, then \textit{something} is not the same.}. If there is more than one state, then there is at least one ``dimension''\footnote{By dimension, again, we just mean something which can differentiate states.} a long which things can differ. \\

A dimension is a set of points, for example time. We do not make the additional claim that such sets of points must be ordered. Each state is the environment at a different point in one or more dimensions. States are ``extended'' along dimensions, meaning no two states can occupy the same points in all dimensions (in other words, states are like the environment perceived from different positions in time or some other dimensions by an omniscient observer). These are all the assumptions we need for our version of pancomputationalism. We don't even need to speculate about any internal structure states might have, or what dimensions might be. Any ``fact'' about a state's internal structure can be defined by its relation to other states. The existence of sets is implied by the existence of states (because there is more than one ``thing'' that exists), and a fact is just the set of states in which it holds (the truth conditions of a ``declarative program''). \\

\noindent \textbf{Universality Claim: } Axioms 1 and 2 hold for every conceivable environment. \\

As truth is defined in existential terms\footnote{This does not mean ``true as interpreted by an omniscient observer'', although it does no harm to think of it that way if it helps intuition. For a practical example, the physical state of a transistor in a computer is a declarative program, but so is everything else that might exist. That's the point of deriving this form of pancomputationalism from first principles. We start from ``things'', and there is more than one state of things when something differs. How this relates to problems of consciousness is addressed elsewhere \cite{bennett2023c,bennett2023d}, and is beyond the scope of this paper.}, there is nothing which is not a fact of the environment, and no environment which does not amount to a set of facts. In other words, this is a minimalist formalism of everything. An environment could be like our own, or not. It could be deterministic, in which case states follow a sequence. It could be non-deterministic, in which case they don't\footnote{In any case, the difference between deterministic and non-deterministic seems meaningless when you consider that DFAs and NFAs are equivalent \cite{rabinscott1959}.}. It could be fantastical with magic and true names. It could even be a world constructed through the subjective experience of its inhabitants. All that matters is that an environment has states, and from the relations between them we obtain the set of all facts. 
Facts as relations between states let us avoid anything like a universal set, because states are otherwise irreducible. An \textbf{aspect} of a state is just a set of facts about that state. With aspects in hand, we can define abstraction.
An aspect is akin to a logical \textbf{statement}. It has a truth value given a state. We also define its \textbf{extension} (def. \ref{abstractionlayer}), which is all other aspects of which it is a part. This implies a heirarchy or ``lattice'' of aspects. Intuitively, an \textbf{abstraction layer} is like a window through which one can view part of the environment. A laptop computer could function as an abstraction layer, as could all or part of the system in which an embodied and embedded organism enacts cognition \cite{thompson2007}. In precise terms an abstraction layer is implied by a \textbf{vocabulary}, which is a set of declarative programs. A vocabulary implies a formal language whose rules are determined by relations between states. 
An abstraction layer implies a set of ``$\mathfrak{v}$-tasks'' (def. \ref{task})\footnote{The notion of task used here descends from the mirror symbol hypothesis \cite{bennettmaruyama2022a,bennett2022a,bennett2023d}, however it is complemented by thematically similar research defining tasks in relation to machine learning and biology \cite{eberding2020,cao2024}.}, each of which is behaviour that defines a system. The formalisation of policies as causal intermediaries between inputs and outputs \cite{putnam1967} then develops this into a causal depiction of goal directed behaviour (assuming a policy is implied by the inputs and outputs). For example, an organism could be a policy for a $\mathfrak{v}$-task which is behaviour that organism might enact. Again, we must emphasise this is a first principles approach. We do not assume symbols and Turing machines. Inputs, outputs and policies are all just sets of declarative programs. Whether something is goal directed is determined by the relations between states. 
The environment makes only one sort of \textbf{value} judgement (existence or non-existence), and is otherwise impartial. Goal directed behaviour is a value judgement, so we formalise this impartiality as a \textbf{uniform distribution} over tasks. Finally, we must now define complexity. The claims we make in the rest of this paper pertain to this notion of complexity. \\

\noindent \textbf{Complexity of Extension: } The complexity of an extension is the cardinality of the simplest statement of which it is the extension (see def. \ref{learning}). This is like other formal notions of complexity \cite{kolmogorov1963,rissanen1978}, but facilitates comparison of abstraction layers. 

\subsection{Implications for Complexity}

\begin{proposition}[subjectivity]\label{proof_subjectivity} If there is no abstraction, complexity can always be minimized without improving sample efficiency, regardless of the task.
\end{proposition}
\begin{proof}
In accord with definition \ref{abstractionlayer}, the absence of abstraction means the vocabulary is the set of all declarative programs, meaning $\mathfrak{v} = P$. It follows that for every $l \in L_\mathfrak{v}$ there exists $f \in \mathfrak{v}$ such that $\bigcap l = f$. Statements $l$ and $\{f\}$ are equivalent iff $E_l = E_{\{f\}}$, which is exactly the case here because $\bigcap l = f$. \cite[prop. 1,2]{bennett2023b} shows maximising weakness is necessary and sufficient to maximise the probability of generalisation, which means weakness maximises sample efficiency (is the optimal proxy). This means sample efficiency is determined by the cardinality of extension. For every correct policy $l$ of every task in $\Gamma_\mathfrak{v}$ there exists $f \in \mathfrak{v}$ s.t. $E_l = E_{\{f\}}$. Policy complexity can be minimised regardless weakness, because the simplest representation of every extension is simplicity $1$.\qed \hphantom{}  
\end{proof}
In this sense, complexity is an illusion created by abstraction. In the absence of any particular abstraction, all behaviours (extensions) are implied by statements of the same complexity. To be clear, we are \textit{not} repeating the claim made by others \cite{leike2015} that \textit{if} the interpreter used by a complexity based induction system matches one used to to compute an objective value for complexity, then that induction system will be optimal in the sense of \textit{eventually} learning the correct policy\footnote{To quote verbatim: ``Legg-Hutter intelligence is measured with respect to a fixed UTM. AIXI is the most intelligent policy if it uses the same UTM.'' \cite{leike2015}}. We are claiming that if interpretation is truly objective, then $\mathfrak{v}=P$ and complexity has nothing to do with intelligence\footnote{Intelligence here meaning not just \textit{eventual} generalisation, but the \textit{efficiency} thereof.}. There \textit{is} no objective notion of complexity. However, when we take empirical measurements it is inevitably through an abstraction layer, for which $\mathfrak{v} \neq P$. In that context simpler forms have been observed to generalise more efficiently. This raises the question; what additional assumptions can we make that would explain the correlation?
\subsubsection{Time, space and causal confounding:} We now make the additional assumption that vocabularies are finite. Every aspect of the world in which we exist appears to be spatially extended, meaning no two things occupy the same space at the same time. For the sake of understanding complexity we assume this is true of all environments. We hold that this justifies the assumption of a finite vocabulary, because in our spatially extended environment the amount of information in a bounded system is finite \cite{bekenstein1981}.

\begin{proposition}[confounding]\label{proof_confounding} If the vocabulary is finite, then policy weakness can confound\footnote{$A$ \textit{confounds} $B$ and $C$ when for example $A = ``badly injured''$ \textit{causes} $B = ``died''$ and $C = ``picked up by ambulance''$, and it looks like $C$ causes $B$ because $p(B \mid C) > p(B \mid \lnot C)$, and yet it may be that $p(B\mid C, A) < p(B \mid \lnot C, A)$.}sample efficiency with policy simplicity. \end{proposition}

\begin{proof}
We already have that policy weakness causes sample efficiency, in that it is necessary and sufficient to maximise it in order to maximise sample efficiency.
Continuing from proof 1, in a finite vocabulary, there may not exist $f \in \mathfrak{v}$ s.t. $E_l = E_{\{f\}}$, which means the complexity of all extensions will not be the same. If we choose any vocabulary in which weaker aspects take simpler forms, then simplicity will be correlated with weakness and so will also be correlated with sample efficiency. This means we would choose $\mathfrak{v}$ s.t. for all $a, b \in L_\mathfrak{v}$, the simpler statement has the larger extension, meaning $a <_w b \leftrightarrow a <_d b $. For example, suppose $P = \{a,b,c…\}, \ a=\{1,2,4\}, \ b=\{1,3,4\}, \ \mathfrak{v} = \{a,b\}, \ L_\mathfrak{v}  = \{\{a\},\{b\},\{a,b\}\}$, then it follows $\{a,b\} <_w \{a\}, \ \{a,b\} <_w \{b\}, \ \{a,b\} <_d \{a\}, \ \{a,b\} <_d \{b\}$. \qed 
\end{proof}
\subsubsection{Why confounding tends to occur:}
We now briefly argue that abstraction is goal directed. This means the tasks an abstraction layer tends to represent are those it is best suited to represent, which implies weak constraints take simple forms. 
There are several reasons an abstraction layer is biased toward particular goals, depending upon the context in which we consider complexity. In the case of a computer, a human has specifically designed each abstraction layer to express that which is needed for a purpose. What separates x86 from a higher level abstraction layer like Numpy is that the former has a more general intended purpose, expressing ``weaker'' constraints. We tend to construct abstraction layers to be as versatile as possible whilst satisfying a particular need. More generally natural selection favours adaptation, which means generalisation, which is maximised by preferring weaker policies \cite{bennett2023b}. Biological cognition is not limited to the brain \cite{ciaunica2023}, meaning the mind is not neatly confined within a well defined neurological abstraction layer. Instead the multiscale competency architectures observed in living organisms \cite{levin2024} amount to self organising abstraction layers. Because natural selection favours adaptable organisms, these abstraction layers will be selected to represent the \textit{weakest} policies which constitute fit behaviour (a weaker policy is more adaptable). More generally, we speculate that phase transitions motivate the emergence of self preserving goal directed behaviour, by destroying some physical structures and preserving others.
Such goal directed abstraction \textit{must} minimise the size of vocabularies at higher levels, whilst also maximising the weakness of the policies they can express (two opposing pressures). This is because a larger a vocabulary exponentially increases the space of outputs and policies \cite{bennett2022a}, which may conflict with finite time and space constraints. A larger vocabulary would make inference and learning less tractable (more ``complex'' in the sense of being a more difficult search problem that takes up more time). 
To maximise the weakness of policies in higher levels of abstraction, while minimising the size of the vocabulary in which they're expressed, weaker policies must take simpler forms.

\printbibliography

@misc{bennett2024c,
    author = "Michael Timothy Bennett",
    title = "Meat Meets Machine! Multiscale Competency Enables Causal Learning",
    year = "2024",
    note={Under review}
}

@article{legg2007,
	pages = {391--444},
	title = {Universal Intelligence: A Definition of Machine Intelligence},
	author = {Shane Legg and Marcus Hutter},
	volume = {17},
	journal = {Minds and Machines},
	publisher = {Springer},
	number = {4},
	year = {2007}
}

@misc{chollet2019,
  %doi = {10.48550/ARXIV.1911.01547},
  %url = {arxiv.org/abs/1911.01547},
  author = {Chollet, François},
  keywords = {Artificial Intelligence (cs.AI), FOS: Computer and information sciences, FOS: Computer and information sciences},
  title = {On the Measure of Intelligence},
  publisher = {arXiv},
  year = {2019}
  %copyright = {arXiv.org perpetual, non-exclusive license}
}

@article{solomonoff1978,
  author={Ray Solomonoff},
  journal={IEEE TIT}, 
  title={Complexity-based induction systems: Comparisons and convergence theorems}, 
  year={1978},
  volume={24},
  number={4},
  pages={422–432}
}

@InProceedings{bennett2024a,
    author = "Michael Timothy Bennett",
    title = "Computational Dualism and Objective Superintelligence",
    booktitle="Artificial General Intelligence",
    year="2024",
    publisher="Springer",
    address="",
    editor=""
}

@InProceedings{bennett2023b,
  author="Bennett, Michael Timothy",
editor="",
title="The Optimal Choice of Hypothesis Is the Weakest, Not the Shortest",
booktitle="Artificial General Intelligence",
year="2023",
publisher="Springer",
address="",
pages="42--51"
}

@InProceedings{bennett2023c,
author="Bennett, Michael Timothy",
editor="",
title="Emergent Causality and the Foundation of Consciousness",
booktitle="Artificial General Intelligence",
year="2023",
publisher="Springer",
address="",
pages="52--61"
}

@InProceedings{bennett2023d,
author="Bennett, Michael Timothy",
editor="",
title="On the Computation of Meaning, Language Models and Incomprehensible Horrors",
booktitle="Artificial General Intelligence",
year="2023",
publisher="Springer",
address="",
pages="32--41"
}

@article{rissanen1978,
  title={Modeling By Shortest Data Description*},
  author={Jorma Rissanen},
  journal={Autom.},
  year={1978},
  volume={14},
  pages={465-471}
}

@article{bennettmaruyama2022a,
  author={Bennett, Michael Timothy and Maruyama, Yoshihiro},
  journal={IEEE Transactions on Cognitive and Developmental Systems}, 
  title={Philosophical Specification of Empathetic Ethical Artificial Intelligence}, 
  year={2021},
  volume={14},
  number={2},
  pages={292-300}
}

@book{sober2015, place={Cambridge}, title={Ockham's Razors: A User's Manual}, 
%DOI={10.1017/CBO9781107705937}, 
publisher={Cambridge Uni. Press}, author={Sober, Elliott}, year={2015}}

@InProceedings{bennett2022a,
    author="Bennett, Michael Timothy",
    %editor="Goertzel, B. and Ikl{\'e}, M. and Potapov, A.",
    title="Symbol Emergence and the Solutions to Any Task",
    booktitle="Artificial General Intelligence",
    year="2022",
    publisher="Springer",
    address="Cham",
    pages="30--40"%,
    %isbn="978-3-030-93758-4"
}

@book{hutter2010,
    author = {Hutter, Marcus},
    title = {Universal Artificial Intelligence: Sequential Decisions Based on Algorithmic Probability},
    year = {2010},
    %isbn = {3642060528},
    publisher = {Springer-Verlag},
    address = {Berlin, Heidelberg}
}

@phdthesis{legg2008,
	author        = "Legg, Shane",
	title         = "Machine Super Intelligence",
	school     = "Uni. of Lugano",
	year          = "2008"
}

@article{leike2015,
  author={Leike, J. and Hutter, M.},
  journal={Proceedings of The 28th COLT, PMLR}, 
  title={Bad Universal Priors and Notions of Optimality},
  year={2015},
  pages={1244-1259}
}

@book{thompson2007,
	author = {Evan Thompson},
	title = {Mind in Life: Biology, Phenomenology, and the Sciences of Mind},
	publisher = {Harvard University Press},
	address = {Cambridge MA},
	year = {2007}
}

@article{kolmogorov1963,
  author={Kolmogorov, A.N.},
  journal={Sankhya: The Indian Journal of Statistics}, 
  title={On tables of random numbers},
  year={1963},
  volume={A},
  pages={369-376}
}

@article{heylighen2023b,
    author = {Heylighen, Francis},
    title = "{The meaning and origin of goal-directedness: a dynamical systems perspective}",
    volume = {139},
    number = {4},
    pages = {370-387},
    year = {2022},
    month = {06},
    journal = {BJLS}
    %journal = {Biological Journal of the Linnean Society}
}

@InCollection{sep-information-entropy,
	author       =	{Maroney, Owen},
	title        =	{{Information Processing and Thermodynamic Entropy}},
	booktitle    =	{The {Stanford} Encyclopedia of Philosophy},
	year         =	{2009}
    %editor       =	{Edward N. Zalta},
	%howpublished =	{\url{https://plato.stanford.edu/archives/fall2009/entries/information-entropy/}},
    %publisher    =	{Metaphysics Research Lab, Stanford University}
}

@inproceedings{
deletang2024,
title={Language Modeling Is Compression},
author={Gregoire Deletang and Anian Ruoss and Paul-Ambroise Duquenne and Elliot Catt and Tim Genewein and Christopher Mattern and Jordi Grau-Moya and Li Kevin Wenliang and Matthew Aitchison and Laurent Orseau and Marcus Hutter and Joel Veness},
booktitle={The Twelfth International Conference on Learning Representations},
year={2024}
}

@Article{gefter2014,
author={Gefter, Amanda},
title={Theoretical physics: Complexity on the horizon},
journal={Nature},
year={2014},
month={May},
day={01},
volume={509},
number={7502},
pages={552-553}
}

@book{barnsley1993a,
author = {Michael F. Barnsley},
editor = {},
title = {Fractals Everywhere},
publisher = {Academic Press},
edition = {2nd ed.},
year = {1993}
}

@article{heylighen2008,
author = {Heylighen, Francis},
year = {2008},
month = {01},
pages = {},
title = {Complexity and Self-organization},
journal = {Encyclopedia of Library and Information Sciences}
}

@article{susskind2014,
    author = "Susskind, Leonard",
    title = "{Computational Complexity and Black Hole Horizons}",
    journal = "Fortsch. Phys.",
    volume = "64",
    pages = "24--43",
    year = "2016"
}

@incollection{putnam1967,
	author = {Hilary Putnam},
	booktitle = {Art, mind, and religion},
	pages = {37--48},
	publisher = {Uni. of Pittsburgh Press},
	title = {Psychological Predicates},
	year = {1967}
}

@article{bekenstein1981,
  title = {Universal upper bound on the entropy-to-energy ratio for bounded systems},
  author = {Bekenstein, Jacob D.},
  journal = {Phys. Rev. D},
  volume = {23},
  issue = {2},
  pages = {287--298},
  numpages = {0},
  year = {1981},
  month = {Jan},
  publisher = {American Physical Society}}

@Article{levin2024,
author={McMillen, Patrick
and Levin, Michael},
title={Collective intelligence: A unifying concept for integrating biology across scales and substrates},
journal={Communications Biology},
year={2024},
month={Mar},
day={28},
volume={7},
number={1},
pages={378}
}

@ARTICLE{rabinscott1959,
  author={Rabin, M. O. and Scott, D.},
  journal={IBM Journal of Research and Development}, 
  title={Finite Automata and Their Decision Problems}, 
  year={1959},
  volume={3},
  number={2}}

@book{piccinini2015,
    author = {Piccinini, Gualtiero},
    title = "{Physical Computation: A Mechanistic Account}",
    publisher = {Oxford University Press},
    year = {2015},
    month = {06}
}

@ARTICLE{ciaunica2023,
AUTHOR={Ciaunica, Anna and Shmeleva, Evgeniya V. and Levin, Michael},   
TITLE={The brain is not mental! coupling neuronal and immune cellular processing in human organisms},      
JOURNAL={Frontiers in Integrative Neuroscience},      
VOLUME={17},           
YEAR={2023}
}

@inproceedings{eberding2020,
  title={SAGE: Task-Environment Platform for Autonomy and Generality Evaluation},
  author={Eberding, Leonard M and Sheikhlar, Arash and Th{\'o}risson, Kristinn R}
}

@article{cao2024,
title = {Explanatory models in neuroscience, Part 2: Functional intelligibility and the contravariance principle},
journal = {Cognitive Systems Research},
volume = {85},
pages = {101200},
year = {2024},
author = {Rosa Cao and Daniel Yamins}
}
\end{document}